\documentclass{article}

\usepackage{amsthm,cite,enumerate,amsthm}
\usepackage{amsmath}
\usepackage{amssymb}

\usepackage{graphicx,epstopdf}

\usepackage{lipsum} 

\usepackage[sc]{mathpazo} 
\usepackage[T1]{fontenc} 
\usepackage{microtype} 

\usepackage[hmarginratio=1:1,top=32mm,columnsep=20pt]{geometry} 
\usepackage{multicol} 
\usepackage[hang, small,labelfont=bf,up,textfont=it,up]{caption} 
\usepackage{subcaption}
\usepackage{booktabs} 
\usepackage{float} 
\usepackage{hyperref} 
\usepackage{algorithm,algorithmicx,algpseudocode}
\usepackage{lettrine} 
\usepackage{paralist} 
 
\usepackage{abstract} 

\usepackage{titlesec} 
\renewcommand\thesection{\Roman{section}} 
\renewcommand\thesubsection{\Roman{subsection}} 
\titleformat{\section}[block]{\large\scshape\centering}{\thesection.}{1em}{} 
\titleformat{\subsection}[block]{\large}{\thesubsection.}{1em}{} 

\usepackage{fancyhdr} 
 \numberwithin{equation}{section}
\newtheorem{thm}{Theorem}[section]

\theoremstyle{definition}

\theoremstyle{definition}
\newtheorem{assump}{Assumption}

\theoremstyle{definition}
\newtheorem{lem}{Lemma}

\theoremstyle{definition}
\newtheorem{cor}{Corollary}

\theoremstyle{definition}
\newtheorem{exmp}{EXAMPLE}

\theoremstyle{definition}

\theoremstyle{definition}
\newtheorem{prop}{Proposition}

\newtheorem{selfdef}{Definition}

\makeatletter
\renewcommand\section{\@startsection {section}{1}{\z@}%
                                   {-3.5ex \@plus -1ex \@minus -.2ex}%
                                   {2.3ex \@plus.2ex}%
                                   {\normalfont\scshape}}
\makeatother

\title{\vspace{-15mm}\fontsize{24pt}{10pt}\selectfont\textbf{Optimal Schatten-q and Ky-Fan-k Norm Rate of Low Rank Matrix Estimation}} 

\author{
\large
\textsc{Dong Xia}\thanks{Thanks Vladimir Koltchinskii for refering this problem to me.}\\[2mm] 
\normalsize Georgia Institute of Technology \\ 
\normalsize \href{mailto:dxia7@math.gatech.edu}{dxia7@math.gatech.edu} 
\vspace{-5mm}
}
\date{}

\begin{document}

\maketitle 



\begin{abstract}
In this paper, we consider low rank matrix estimation using either matrix-version Dantzig Selector $\hat{A}_{\lambda}^d$ as in (\ref{dantmodel}) or matrix-version LASSO estimator $\hat{A}_{\lambda}^L$ as in (\ref{lassomodel}). We consider sub-Gaussian measurements, $i.e.$, the measurements
 $X_1,\ldots,X_n\in\mathbb{R}^{m\times m}$ have $i.i.d.$ sub-Gaussian entries. Suppose $\textrm{rank}(A_0)=r$. We proved that, when $n\geq Cm[r^2\vee r\log(m)\log(n)]$ for some $C>0$, both $\hat{A}_{\lambda}^d$ and $\hat{A}_{\lambda}^L$ can obtain optimal upper bounds(except some logarithmic terms) for estimation 
 accuracy under spectral norm.
 By applying metric entropy of Grassmann manifolds, we construct (near) matching minimax lower bound for estimation accuracy under spectral norm. Note that, Candes and Plan\cite{candesplan}, Negahban and Wainwright\cite{negahban},
 Rohde and Tsybakov\cite{rohdetsybakov} proved optimal upper bound for estimation accuracy under Frobenius norm as long as $n\geq Cmr$ for some constant $C>0$. We also give upper bounds and matching minimax lower bound(except some logarithmic terms) for estimation accuracy under Schatten-q norm for
 every $1\leq q\leq\infty$. As a direct corollary, we show both upper bounds and minimax lower bounds of estimation accuracy under Ky-Fan-k norms for every $1\leq k\leq m$. Our minimax lower bounds are similar to those given in an earlier paper by Ma and Wu\cite{mawu}.
\noindent 

\end{abstract}





\section{Introduction and an overview of main results}
\label{introsect}
Low rank matrix estimation has been studied for several years in the literatures, such as Candes and Plan\cite{candesplan}, Koltchinskii\cite{koltchinskii2011},
Koltchinskii\cite{koltchinskiisharp} and Klopp\cite{kloppmatrixlasso} with references therein. In the general settings, we have independent 
pairs of measurements and outputs, $(X_1,Y_1),\ldots,(X_n,Y_n)\in(\mathbb{R}^{m_1\times m_2},\mathbb{R})$ which are related to an unkown matrix $A_0\in\mathbb{R}^{m_1\times m_2}$.
We assume $A_0$ has low rank, $i.e.$, $r=\text{rank}(A_0)\ll(m_1\wedge m_2)$. 
The observations $(X_j,Y_j),j=1,\ldots,n$ satisfy the trace regression model:
\begin{equation}
 \label{trmodel}
 Y_j=\left<A_0,X_j\right>+\xi_j,\quad j=1,\ldots,n
\end{equation}
where $\xi_j,j=1,\ldots,n$ are $i.i.d.$ zero-mean random noises with variance $\mathbb{E}\xi^2=\sigma_{\xi}^2<\infty$. In this paper, we only consider sub-Gaussian noise, $i.e., |\xi|_{\psi_2}\lesssim \sigma_{\xi}^2$.
The meaning of $||\cdot||_{\psi_2}$ and $\lesssim$ will be introduced later.
$\left<A,B\right>$ is used as notation for $\text{Tr}(A^TB)$ for any $A,B\in\mathbb{R}^{m_1\times m_2}$. The task is to estimate
$A_0$ based on the collected data $(X_j,Y_j),j=1,\ldots,n$. Let $\mathcal{Y}:=\left(Y_1,\ldots,Y_n\right)^T\in\mathbb{R}^{n}$.\\
We use $\Pi$ to denote the distribution of $i.i.d.$ measurements $X_j,j=1,\ldots,n$, which are sampled from the measurements set $\mathcal{M}$. Distribution based dot product and $L_2$-norm
are defined as
\begin{equation}
 \left<A,B\right>_{L_2(\Pi)}:=\mathbb{E}\left<A,X\right>\left<B,X\right>
\end{equation}
and
\begin{equation}
 ||A||_{L_2(\Pi)}^2:=\mathbb{E}\left<A,X\right>^2
\end{equation}
Several well-konwn measurements $\mathcal{M}$ and $\Pi$ have been studied in the literatures, such as
\begin{exmp}
\label{mcexmp}
 {\bf Matrix Completion} In this situation, $\Pi$ denotes some distribution on the set 
 \begin{equation}
  \mathcal{M}=\left\{e_j(m_1)\otimes e_k(m_2),j=1,\ldots,m_1,k=1,\ldots,m_2\right\}
 \end{equation}
where $e_j(m)$ denotes the $j$-th canonical basis vector in $\mathbb{R}^{m}$. Most literatures considered $\Pi$ as a uniform distribution 
on the set $\mathcal{M}$, see Koltchinskii\cite{koltchinskii1}, Koltchinskii et al.\cite{koltchinskii2} and Rohde and Tsybakov\cite{rohdetsybakov}.
Lounici\cite{lounicispectral} and Klopp\cite{kloppmatrixlasso} studied general sampling on $\mathcal{X}$ instead.
Under the assumption of uniform distribution, the task means to estimate $A_0$ from randomly observed entries of $A_0$ which are corrupted with noises.
Rohde and Tsybakov\cite{rohdetsybakov} also considered sampling without replacement from $\mathcal{M}$, i.e. $X_1,\ldots,X_n$ are different
from each other. A remark is that when $\Pi$ denotes the uniform distribution on $\mathcal{M}$, we have $||A||_{L_2(\Pi)}^2=\frac{1}{m_1m_2}||A||_2^2$ 
and $\left<A,B\right>_{L_2(\Pi)}=\frac{1}{m_1m_2}\left<A,B\right>$.
\end{exmp}

\begin{exmp}
\label{subgaussianexmp}
 {\bf sub-Gaussian Design} In this situation, $X_j,j=1,\ldots,n$ are $i.i.d.$ designed matrices. The entries of every $X_j$ are all $i.i.d.$ sub-Gaussians.
 A real-valued random variable $x$ is said to be {\it sub-Gaussian} with parameter $b>0$ if it has the property that for every $t\in\mathbb{R}$ one has:
 $\mathbb{E}e^{tx}\leq e^{b^2t^2/2}$.
 In Gaussian and Rademacher cases, $||A||_{L_2(\Pi)}=||A||_2$ and $\left<A,B\right>_{L_2(\Pi)}=\left<A,B\right>$. Koltchinskii\cite{koltchinskii1}
 studied the sub-Gaussian measurements for estimating density matrices in quantum state tomography. Gaussian measurements are widely discussed 
 in compressed sensing for the reason that, with high probability, Gaussian random sampling operator satisfies the {\it Restricted Isometry Property}, which will
 be introduced in Section~\ref{defsect}.
 Interested readers can read Baraniuk et al.\cite{BDDW}, Candes et al.\cite{CRT2006}.
\end{exmp}

\begin{exmp}
 \label{rankoneproj}
 {\bf Rank One Projection} As described in Cai and Zhang\cite{caizhang2013}, both Example~\ref{mcexmp} and Example~\ref{subgaussianexmp} have disadvantages. 
 Under the matrix completion model, in order to get a stable estimation of matrix $A_0$, as pointed out by Candes and Recht\cite{candesrecht}, Gross\cite{gross2011},
 additional structral assumptions are needed. Actually, it is impossible to recover spiked matrices under matrix completion model. However, under sub-Gaussian sampling,
 every measurements $X_j,j=1,\ldots,n$ require $\mathcal{O}(m_1m_2)$ bytes of space for storage, which will be huge when $m$ is large. Therefore, Cai and Zhang\cite{caizhang2013}
 proposed the rank one projection, $X_j=\alpha_j^T\beta_j,j=1,\ldots,n$, where $\alpha_j,j=1,\ldots,n$ and $\beta_j,j=1,\ldots,n$ are $i.i.d.$ sub-Gaussian vectors. They proved that
 under rank one projection, one is able to construct a stable estimator without addition structral assumptions. In addition, only $\mathcal{O}(m_1+m_2)$ bytes of space are
 needed for storage of every $X_j,j=1,\ldots,n$.
\end{exmp}

{\it sub-Gaussian Design}. In this paper, we only consider sub-Gaussian design with introduction similar to Koltchinskii\cite{koltchinskii1}. More precisely, we assume that the distribution $\Pi$ satisfies that, for some constant $b_0>0$ such that for any $A\in\mathbb{R}^{m_1\times m_2}$,
$\left<A,X\right>$ is a sub-Gaussian random variable with parameter $b_0||A||_{L_2(\Pi)}$. This implies that $\mathbb{E}X=0$ and, for some constant $b_1>0$,
\begin{equation}
 ||\left<A,X\right>||_{\psi_2}\leq b_1||A||_{L_2(\Pi)}, \quad \forall A\in\mathbb{R}^{m_1\times m_2}.
\end{equation}
In addition, assume that, for some constant $b_2>0$,
\begin{equation}
 ||A||_{L_2(\Pi)}=||\left<A,X\right>||_{L_2(\Pi)}\leq b_2||A||_2,\quad \forall A\in\mathbb{R}^{m_1\times m_2}.
\end{equation}
A random matrix $X$ satisfying the above conditions will be called a {\it sub-Gaussian} matrix. Moreover, if $X$ also satisfies the condition
\begin{equation}
 ||A||_{L_2(\Pi)}=||A||_2,\quad \forall A\in\mathbb{R}^{m_1\times m_2}
\end{equation}
then it will be called an {\it isotropic sub-Guassian} matrix. As was mentioned in Example~\ref{subgaussianexmp}, Gaussian and Rademacher random matrices belong to the class of {\it isostropic sub-Gaussian} matrices.
It easily follows from the basic properties of Orlicz norms, van der Vaart and Wellner\cite{vaartwellner}, that for sub-Gaussian matrices $||A||_{L_p(\Pi)}=\mathbb{E}^{1/p}\left<A,X\right>^p\leq c_pb_1b_2||A||_2$
and $||A||_{\psi_1}:=||\left<A,X\right>||_{\psi_1}\leq cb_1b_2||A||_2, A\in\mathbb{R}^{m_1\times m_2},p\geq 1$ for some universal constants $c_p>0,c>0$.\\

To simplify our expressions, $W.L.O.G.$, we assume $m_1=m_2=m$. Let $\mathcal{X}$ denotes the following linear map:
\begin{equation}
 \forall A\in\mathbb{R}^{m\times m}, \mathcal{X}(A)=\left(\left<A,X_1\right>,\ldots,\left<A,X_n\right>\right)^T\in\mathbb{R}^n
\end{equation}
Therefore, when $X_j,j=1,\ldots,n$ are random matrices, $\mathcal{X}(A)$ is a random vector in $\mathbb{R}^n$ for every $A\in\mathbb{R}^{m\times m}$. The adjoint operator $\mathcal{X}^{\star}$ is given as
\begin{equation}
 \forall U\in\mathbb{R}^n, \mathcal{X}^{\star}(U)=\sum\limits_{j=1}^n \left<U,X_j\right>X_j
\end{equation}
Now we introduce some notations we will use in this paper. For $\forall A\in\mathbb{R}^{m\times m}, $Let $||A||_q$ denotes the Schatten-q norm for every $q\geq 1$, $i.e., ||A||_q^q=\sum\limits_{j=1}^m\sigma_j^q(A)$,
where we assumed that $A$ has singular value decomposition as $A=\sum\limits_{j=1}^m\sigma_j(A)u_j\otimes v_j$ with $\sigma_{j}(A),j=1,\ldots,m$ arranged in non-increasing order.
Therefore, $||\cdot||_2$ is Frobenius norm, $||\cdot||_1$ as nuclear norm and $||\cdot||_{\infty}$ as spectral norm. Another similar norms are Ky-Fan norms. Given any $1\leq k\leq m$, 
the Ky-Fan-$k$ norm is defined as $||A||_{F_k}:=\sum\limits_{j=1}^k\sigma_j(A),\quad\forall A\in\mathbb{R}^{m\times m}$. As described in Tao\cite[Chapter 2]{taorandmat}, $||\cdot||_{F_k}:\mathbb{R}^{m\times m}\to\mathbb{R}$ is
a convex function for every $1\leq k\leq m$.
For a vector $v\in\mathbb{R}^n$, we use $|v|_{l_2}$
to denote the $l_2-$norm, $i.e., |v|_{l_2}^2=\sum\limits_{j=1}^nv_j^2$.\\ 
We use $A_{\max(r)}$ to denote $A_{\max(r)}:=\sum\limits_{j=1}^r u_j(A)u_j\otimes v_j$. We also define $A_{-\max(r)}:=A-A_{\max(r)}$.
A cone $\mathcal{C}(r,\beta)$ is defined as 
$$
\mathcal{C}(r,\beta):=\left\{A\in\mathbb{R}^{m\times m}, ||A_{-\max(r)}||_1\leq \beta||A_{\max(r)}||_1\right\}.
$$
Let $W:=\sum\limits_{j=1}^n\xi_jX_j$. We use $x\gtrsim y$ to denote that $x\geq cy$ for some constant $c>0$. Similar notation is $\lesssim$. Let $\mathcal{A}_r$ denotes the set of all matrices $A\in\mathbb{R}^{m\times m}$ with $\text{rank}(A)\leq r$. \\
Based on the data $(X_1,Y_1),\ldots,(X_n,Y_n)$, several estimators of $A_0$ have been proposed. The following two estimators are well-studied in the literature. The first one is  matrix-version LASSO estimator:
\begin{equation}
\label{lassomodel}
 \hat{A}_{\lambda}^{L}:=\underset{A\in\mathbb{R}^{m\times m}}{\arg\min} \sum\limits_{j=1}^n\left(\left<A,X_j\right>-Y_j\right)^2+\lambda\left|\left|A\right|\right|_1,
\end{equation}
where $||\cdot||_1$ is used as a convex surrogate for $\text{rank}(\cdot)$ to "promote" low rank solution. Readers can refer to Koltchinskii\cite{koltchinskii2011}, Rohde and Tsybakov\cite{rohdetsybakov},
and Klopp\cite{kloppmatrixlasso} for more details. Another estimator is Dantzig Selector
\begin{equation}
 \label{dantmodel}
 \hat{A}_{\lambda}^{d}:=\underset{A\in\mathbb{R}^{m\times m}}{\arg\min}\left\{||A||_1: \left|\left|\mathcal{X}^{\star}(\mathcal{X}A-\mathcal{Y})\right|\right|_{\infty}\leq \lambda\right\}
\end{equation}
Candes and Plan\cite{candesplan} proved that, under Gaussian measurements, when $n\geq Cmr$ for some constant $C$ and $|\xi|_{\psi_2}\lesssim \sigma_{\xi}$, if we choose $\lambda=C_1\sigma_{\xi}\sqrt{nm\log(m)}$ for some constant $C_1>0$, then
$||\hat{A}_{\lambda}^d-A_0||_2^2\leq C'\frac{mr\sigma_{\xi}^2\log(m)}{n}$ and $||\hat{A}_{\lambda}^L-A_0||_2^2\leq C'\frac{mr\sigma_{\xi}^2\log(m)}{n}$ with high probability for some universal constant $C'>0$.
They also showed that these upper bounds are optimal.\\
In addition, Lounici\cite{lounicispectral}, Koltchinskii and Lounici et. al.\cite{koltchinskii2}
considered the following modified matrix LASSO estimator:
\begin{equation}
 \label{modlassomodel}
 \hat{A}_{\lambda}^{mL}:=\underset{A\in\mathbb{R}^{m\times m}}{\arg\min} ||A||_{L_2(\Pi)}^2-\left<A,\frac{1}{n}\sum\limits_{j=1}^nY_jX_j\right>+\lambda\left|\left|A\right|\right|_1.
\end{equation}
 Under (near) matrix completion model and certain assumptions, optimal upper bounds (except some logarithmic terms) for estimation accuracy under both the spectral norms, $i.e., ||\hat{A}_{\lambda}^{mL}-A_0||_{\infty}$ and
 Frobenius norm $i.e., ||\hat{A}_{\lambda}^{mL}-A_0||_2$, are obtained in \cite{lounicispectral} and \cite{koltchinskii2}.\\
However, there are few results about estimation accuracy under the spectral norm, $i.e., ||\hat{A}_{\lambda}^d-A_0||_{\infty}$ or $||\hat{A}_{\lambda}^L-A_0||_{\infty}$. 
In this paper, we will give optimal(except some logarithmic terms) upper bounds for them under sub-Gaussian measurements. Unlike \cite{candesplan}, our analysis requires $n\geq Cm[r^2\vee r\log(m)\log(n)]$ for some constant $C>0$ which requires higher order of $r$. 
The idea of the proof is similar to Lounici\cite{lounicisup}. We state our main results as follows, some notations will be described in Section~\ref{defsect}.
\begin{thm}
\label{intromainthm1}
 Suppose $\Pi$ is a sub-Gaussian distribution and $n\geq Cm[r^2\vee r\log(m)\log(n)]$ for some $C>0$ and $|\xi|_{\psi_2}\lesssim \sigma_{\xi}$, if $\lambda\gtrsim C_2\sigma_{\xi}\sqrt{nm\log m}$ for some $C_2>0$, then
 there exists some constant $C_1>0$ such that with probability at least $1-\frac{4}{m}$,
 \begin{equation}
  \left|\left|\hat{A}_{\lambda}-A_0\right|\right|_{\infty}\leq C_1\sigma_{\xi}\sqrt{\frac{m\log m}{n}}
 \end{equation}
 where $\hat{A}_{\lambda}$ can be $\hat{A}_{\lambda}^d$ and $\hat{A}_{\lambda}^L$. $C_1$ contains some constants related to distribution $\Pi$.
\end{thm}
In fact, we can prove a further result by applying interpolation inequality.
\begin{thm}
\label{intromainthm2}
 Under the same assumptions of Theorem~\ref{intromainthm1}, there exists some constat $C_1>0$ such that, for every $1\leq q\leq\infty$, with probability at least $1-\frac{4}{m}$,
 \begin{equation}
  \left|\left|\hat{A}_{\lambda}-A_0\right|\right|_{q}\leq C_1\sqrt{\frac{m\log(m)}{n}}\sigma_{\xi}r^{1/q}
 \end{equation}
 and for any integer $1\leq k\leq m$,
 \begin{equation}
  \left|\left|\hat{A}_{\lambda}-A_0\right|\right|_{F_k}\leq C_1(k\wedge r)\sqrt{\frac{m\log(m)}{n}}\sigma_{\xi}
 \end{equation}

  where $\hat{A}_{\lambda}$ can be $\hat{A}_{\lambda}^d$ and $\hat{A}_{\lambda}^L$. $C_1$ contains some constants related to distribution $\Pi$.
\end{thm}
The following Theorem shows that the previous bounds in Theorem~\ref{intromainthm1} and Theorem~\ref{intromainthm2} are optimal in the minimax sense, except some logarithmic terms.
\begin{thm}
\label{intromainthm3}
 Suppose the $i.i.d.$ noise $\xi_1,\ldots,\xi_n\sim\mathcal{N}(0,\sigma^2_{\xi})$ and $\Pi$ denotes sub-Gaussian distribution,$2r\leq m$, then there exists some universal constant $c>0$ and $c'>0$ such that
 for every $1\leq q\leq \infty$,
 \begin{equation}
  \underset{\hat{A}}{\inf}\underset{A\in\mathcal{A}_r}{\sup}\mathbb{P}_A\left(||\hat{A}-A||_q\geq c\sigma_{\xi}r^{1/q}\sqrt{\frac{m}{n}}\right)\geq c'
 \end{equation}
 and for any integer $1\leq k\leq m$,
 \begin{equation}
  \underset{\hat{A}}{\inf}\underset{A\in\mathcal{A}_r}{\sup}\mathbb{P}_A\left(||\hat{A}-A||_{F_k}\geq c\sigma_{\xi}(k\wedge r)\sqrt{\frac{m}{n}}\right)\geq c'
 \end{equation}
 where $\mathbb{P}_A$ denotes the joint distribution of $(X_1,Y_1),\ldots,(X_n,Y_n)$ when $Y_j=\left<A,X_j\right>+\xi_j,j=1,\ldots,n$.
\end{thm}
The proof of Theorem~\ref{intromainthm3} applied the metric entropy bounds of Grassmann manifolds, introduced in Section~\ref{defsect}.\\
The rest of the paper is organized as follows. In Section~\ref{defsect}, we introduce some preliminaries which will be needed in our proof, such as {\it Restricted Isometry Property with constant $\delta_r\in(0,1)$}, {\it Empirical Process Bounds},
{\it metric entropy bounds of Grassmann manifolds $\mathcal{G}_{m,k}$} and {\it rotation invariance} of sub-Gaussians. In Section~\ref{generalsect}, we will prove the upper bound of estimation accruacy under Schatten-q norm for every $1\leq q\leq \infty$,
as long as the assumption that $\delta_2\lesssim \frac{1}{r}$ holds. In Section~\ref{subgaussiansect}, we will prove that, under sub-Gaussian sampling, the random operator $\mathcal{X}$ satisfies the assumption $\delta_2\lesssim \frac{1}{r}$ with high probability
when $n\geq Cm[r^2\vee r\log(m)\log(n)]$ for some $C>0$. In Section~\ref{minimaxsect}, by applying the metric entropy bounds, we can construct a set $\mathcal{A}\subset\mathcal{A}_r$ such that
the minimax lower bounds in Theorem~\ref{intromainthm3} holds. In Section~\ref{numericsect}, results of numerical simulations will be displayed.

\section{Definitions and Preliminaries}
\label{defsect}
In this section, we will introduce some definitions and preliminaries we need for our proof. \\
{\it Sub-differentials of nuclear norm}. Given $A\in\mathbb{R}^{m\times m},\text{rank}(A)=r$ with singular value decomposition, $A=U\Sigma V^T$ where $U\in\mathbb{R}^{m\times r}$, $\Sigma\in\mathbb{R}^{r\times r}$ and $V\in\mathbb{R}^{m\times r}$,
the sub-differential of the convex function $A\to||A||_1$ is given as the following set, Watson\cite{watson}:
\begin{equation}
 \partial||A||_1:=\left\{UV^T+\mathcal{P}_{S_1^{\perp}}\Phi\mathcal{P}_{S_2^{\perp}}\in\mathbb{R}^{m\times m}: ||\Phi||_{\infty}\leq 1\right\}
\end{equation}
where $S_1$ denotes the linear span of $\left\{u_1,\ldots,u_r\right\}$ and $S_2$ denotes the linear span of $\left\{v_1,\ldots,v_r\right\}$. It is easy to see that 
for any $\Lambda\in\partial||A||_1$, we have $||\Lambda||_{\infty}=1$ as long as $A\neq 0$.\\
\newline
{\it Restricted isometry property}, initially introduced by Candes and Plan\cite{candesplan}, is defined as follows:
\begin{selfdef}
 For each integer $r=1,2,\ldots,m$, the isometry constant $\delta_r$ of $\mathcal{X}$ is the smallest quantity such that
 \begin{equation}
  (1-\delta_r)||A||_2^2\leq \frac{1}{n}||\mathcal{X}(A)||_{l_2}^2\leq (1+\delta_r)||A||_2^2
 \end{equation}
 holds for all matrices $A\in\mathbb{R}^{m\times m}$ of rank at most $r$.
\end{selfdef}
We say that $\mathcal{X}$ satisfies the RIP with constant $\delta_r$ at rank $r$ if $\delta_r$ is bounded by a sufficiently small constant between $0$ and $1$. 
We proved that RIP holds with high probability under sub-Gaussian measurements in Section~\ref{subgaussiansect}. Our proof here is different from \cite{candesplan}. We obtain an upper bound for the empirical process $\underset{||A||_2^2\leq[1/2,2],\text{rank}(A)\leq r}{\sup}\left|\frac{1}{n}\sum\limits_{j=1}^n\left<A,X_j\right>^2-\mathbb{E}\left<A,X\right>^2\right|$ while \cite{candesplan} applied an $\epsilon-$net 
argument to prove the RIP under Gaussian measurements. \cite{candesplan} proved RIP with higher probability than ours, however $\epsilon-$net argument is more complicated and cannot directly be applied to sub-Gaussian measurements. We will see later that, when we have
an sharp upper bound of $\delta_2$, we are able to derive an optimal upper bound for estimation accuracy under spectral norm. The following lemma is also due to \cite{candesplan}. We repeat their proof for
self-containment.
\begin{lem}
\label{corlem}
 For integer $r,r'=1,2,\ldots,m$ and $r+r'\leq m$, suppose Assumption~\ref{corassump} holds for $\delta_{r+r'}$, then for any matrix $A\in\mathbb{R}^{m\times m}$ and $B\in\mathbb{R}^{m\times m}$ obeying $\left<A,B\right>=0$ with $\text{rank}(A)\leq r$ and $\text{rank}(B)\leq r'$, we have
 \begin{equation}
 \label{corineq}
  \frac{1}{n}\left|\left<\mathcal{X}(A),\mathcal{X}(B)\right>\right|\leq \delta_{r+r'}||A||_2||B||_2
 \end{equation}
\end{lem}
\begin{proof}
 We can certainly assume that $||A||_2=||B||_2=1$. Otherwise, we can just rescale $A$ and $B$, since (\ref{corineq})
 is invariant by scaling. Then according to definition of $\delta_{r+r'}$, we have
 \begin{equation}
  (1-\delta_{r+r'})||A\pm B||_2^2\leq \frac{1}{n}||\mathcal{X}(A\pm B)||_{l_2}^2\leq (1+\delta_{r+r'})||A\pm B||_2^2
 \end{equation}
According to these two inequalities, it is easy to get that
\begin{equation}
 \frac{4}{n}\left|\left<\mathcal{X}(A),\mathcal{X}(B)\right>\right|\leq 4\delta_{r+r'}
\end{equation}
\end{proof}

{\it Empirical Process Bounds}. Our techniques of proof requires some inequalities of empirical process indexed by a class of measureable functions $\mathcal{F}$ defined on an arbitrary measureable space
$(S,\mathcal{A})$. The following introductions are similar to Koltchinskii\cite{koltchinskii1}. Let $X,X_1,\ldots,X_n$ be $i.i.d.$ random variables in $(S,\mathcal{A})$ with common distribution $P$. One of these inequalities is the Adamczak's version of Talagrand inequality,\cite{adamczak}.
Let $F(X)\geq \underset{f\in\mathcal{F}}{\sup}|f(X)|,X\in S$, be an envelope of the class.
Then, there exists a constant $K>0$ such that for all $t>0$ with probability at least $1-e^{-t}$,
\begin{equation}
\label{adamczakineq}
\begin{split}
 \underset{f\in\mathcal{F}}{\sup}&\left|\frac{1}{n}\sum\limits_{j=1}^nf(X_j)-\mathbb{E}f(X)\right|\\
 \leq& K\left[\mathbb{E}\underset{f\in\mathcal{F}}{\sup}\left|\frac{1}{n}\sum\limits_{j=1}^nf(X_j)-\mathbb{E}f(X)\right|+\sigma_{\mathcal{F}}\sqrt{\frac{t}{n}}+\left|\left|\max_{1\leq j\leq n} |F(X_j)|\right|\right|_{\psi_1}\frac{t}{n}\right]
 \end{split}
\end{equation}
where $\sigma^2_{\mathcal{F}}:=\underset{f\in\mathcal{F}}{\sup}\text{Var}(f(X))$. For $\forall \alpha\in[1,2]$, $||f||_{\psi_\alpha}:=\underset{}{\inf}\left\{C>0: \int_S\psi_{\alpha}\left(\frac{|f(X)|}{C}\right)dP\leq 1\right\}$,
where $\psi_{\alpha}(t):=e^{t^{\alpha}}-1,t\geq 0$. Usually, $\psi_2$ is related to sub-Gaussian tails and $\psi_1$ is related to sub-exponential tails. \\
Mendelson\cite{mendelson2010} developed a subtle upper bound on $\mathbb{E}\underset{f\in\mathcal{F}}{\sup}\left|\frac{1}{n}\sum\limits_{j=1}^nf^2(X_j)-\mathbb{E}f^2(X)\right|$ based on
generic chaining bound. Talagrand's generic chaining complexity,\cite{talagrand1996}, of a metric space $(\mathcal{T},d)$ is defined as follows. An admissible sequence $\left\{\Delta_n\right\}_{n\geq 0}$
is an increasing sequence of partitions of $\mathcal{T}$ ($i.e.$ each partition is a refinement of the previous one) such that $\text{card}(\Delta_0)=1$ and $\text{card}(\Delta_n)\leq 2^{2^n},n\geq 1$.
For $t\in\mathcal{T}$, $\Delta_n(t)$ denotes the unique subset in $\Delta_n$ that contains $t$. For a set $B\subset \mathcal{T}$, $D(B)$ denotes its diameter. Then, the generic chaining complexity $\gamma_2(\mathcal{T};d)$ is defined
as
\begin{equation}
 \gamma_2(\mathcal{T};d):=\underset{\{\Delta_n\}_{n\geq 0}}{\inf}\underset{t\in\mathcal{T}}{\sup}\sum\limits_{n\geq 0}2^{n/2}D(\Delta_n(t)),
\end{equation}
where the $\inf$ is taken over all admissible sequences of partitions. Talagrand\cite{talagrand1996} used the generic chaining complexities to characterize the size
of the expected sup-norms of Gaussian processes. Actually, Talagrand\cite{talagrand1996} proved that for a Gaussian process $G_t$ indexed by $t\in\mathcal{T}$, one has
\begin{equation}
 c\gamma_2(\mathcal{T},d)\leq \mathbb{E}\underset{t\in\mathcal{T}}{\sup}G_t\leq C\gamma_2(\mathcal{T},d)
\end{equation}
for some universal constant $c,C>0$.
Similar quantities as $\gamma_2(\mathcal{T},d)$ are also used to control the size of empirical process indexed by a function class $\mathcal{F}$.
Mendelson\cite{mendelson2010} used $\gamma_2(\mathcal{F},\psi_2)$ to control the size of expected emprical process. Suppose $\mathcal{F}$ is a symmetric class, that is,
$f\in\mathcal{F}$ implies $-f\in\mathcal{F}$, and $\mathbb{E}f(X)=0$, for $\forall f\in\mathcal{F}$. Then, for some constant $K>0$,
\begin{equation}
\label{mendelsonineq}
 \mathbb{E}\underset{f\in\mathcal{F}}{\sup}\left|\frac{1}{n}\sum\limits_{j=1}^nf^2(X_j)-\mathbb{E}f^2(X)\right|\leq K\left[\underset{f\in\mathcal{F}}{\sup}||f||_{\psi_1}\frac{\gamma_2(\mathcal{F},\psi_2)}{\sqrt{n}}\bigvee \frac{\gamma_2^2(\mathcal{F},\psi_2)}{n}\right]
\end{equation}
We will apply these empirical bounds to prove strong RIP of $\delta_2$ for sub-Gaussian measurements.\\
\newline
{\it Interpolation Inequality}. For $0<p<q<r\leq \infty$, let $\theta\in[0,1]$ be such that $\frac{\theta}{p}+\frac{1-\theta}{r}=\frac{1}{q}$. Then for all $A\in\mathbb{R}^{m_1\times m_2}$,
\begin{equation}
\label{interpolationineq}
 ||A||_q\leq ||A||_p^{\theta}||A||_r^{1-\theta}
\end{equation}
One proof of this inequality is given in Rohde and Tsybakov\cite{rohdetsybakov}.\\
\newline
{\it Metric entropy of Grassmann manifolds}. The Grassmann manifold $\mathcal{G}_{m,k}$ is the collection of all subspaces with dimension $k$ in $\mathbb{R}^m$. For
any subspace $E\in\mathcal{G}_{m,k}$ we denote by $P_E$ the orthogonal projection onto $E$. For any metric $d:\mathcal{G}_{m,k}\times\mathcal{G}_{m,k}\to\mathbb{R}$, an $\epsilon-$net 
of $\mathcal{G}_{m,k}$ is a subset $\Gamma$ of $\mathcal{G}_{m,k}$ such that for any point $x\in\mathcal{G}_{m,k}$ can be approximated by a point $y\in\Gamma$ such that $d(x,y)<\epsilon$.
The smallest cardinality of an $\epsilon-$net of $\mathcal{G}_{m,k}$ is called the covering number of $\mathcal{G}_{m,k}$ and is denoted by $N(\mathcal{G}_{m,k},d,\epsilon)$. The metric entropy is the function
$\log N(\mathcal{G}_{m,k},d,\cdot)$.\\
For every $1\leq q\leq\infty$, we define the metric $\tau_q:\mathcal{G}_{m,k}\times\mathcal{G}_{m,k}\to\mathbb{R}$ by
\begin{equation}
 \forall E, F\in\mathcal{G}_{m,k}, \tau_q(E,F)=||P_E-P_F||_q
\end{equation}
According to definition of Schatten-q norms, the metric $\tau_q$ is well defined. Pajor\cite{pajor} proved that 
\begin{prop}
 \label{metricprop}
 For any integers $1\leq k\leq m$ such that $k\leq m-k$, for any $1\leq q\leq \infty$ and for every $\epsilon>0$, we have
 \begin{equation}
  \left(\frac{c}{\epsilon}\right)^d\leq N(\mathcal{G}_{m,k},\tau_q,\epsilon k^{1/q})\leq \left(\frac{C}{\epsilon}\right)^d,
 \end{equation}
where $d=k(m-k)$ and $c,C>0$ are universal constants. 
\end{prop}
Given any metric $d(\cdot,\cdot)$ on $\mathcal{G}_{m,k}$, an $\epsilon-$packing is a subset $\tilde{\Gamma}\subset \mathcal{G}_{m,k}$ such that for any $x,y\in\tilde{\Gamma}, x\neq y$, we have $d(x,y)\geq \epsilon$.	
The packing number of $\mathcal{G}_{m,k}$, denoted as $M(\mathcal{G}_{m,k},d,\epsilon)$, is the largest cardinality of an $\epsilon-$packing of $\mathcal{G}_{m,k}$.
One can easily check that,
\begin{equation}
\label{coverpackineq}
 N(\mathcal{G}_{m,k},d,\epsilon)\leq M(\mathcal{G}_{m,k},d,\epsilon)\leq N(\mathcal{G}_{m,k},d,\epsilon/2)
\end{equation}
\newline
{\it Rotation invariance of sub-Gaussians}. The proof of the following lemma can be found in Vershynin\cite[Lemma 5.9]{vershynin2011}.
\begin{lem}
\label{subgaussiannormlem}
 Consider a finite number of independent centered sub-Gaussian random variables $X_j,j=1,\ldots,n$. Then $\sum\limits_{j=1}^n X_j$ is also a centered sub-Gaussian random variables. Moreover,
 \begin{equation}
  \left|\left|\sum\limits_{j=1}^n X_j\right|\right|_{\psi_2}^2\leq C\sum\limits_{j=1}^n||X_j||_{\psi_2}^2
 \end{equation}
 where $C>0$ is a universal constant.
\end{lem}


\section{Spectral norm rate under general settings}
\label{generalsect}
In this section, we will prove the upper bound for estimation accuracy under spectral norm in general settings, as long as certain assumptions are satisfied. 
In the next section, we will show that these assumptions are satisfied with high probability under sub-Gaussian measurements.
The assumption is related to the RIP constant $\delta_2$. It is similar to the Assumption 2 in Lounici\cite{lounicisup}.
\begin{assump}
\label{corassump}
 $\delta_2\leq \frac{1}{\alpha(1+2c_0)r}$ for integer $r\geq 1$ and some constant $\alpha>1$. $c_0$ depends on whether we study $\hat{A}_{\lambda}^d$ or $\hat{A}_{\lambda}^L$.
 Actually we can choose $c_0=1$ for $\hat{A}_{\lambda}^d$ and $c_0=3$ for $\hat{A}_{\lambda}^L$.
\end{assump}
According to Proposition~\ref{empprop}, we see that Assumption~\ref{corassump} holds with probability at least $1-\frac{1}{m}$ as long as $n\geq C_1m[\alpha^2(1+2c_0)^2r^2\vee \alpha(1+2c_0)r\log(m)\log(n)]$ for some $C_1>0$.
Note that we need $n\gtrsim mr^2$ to get an optimal upper bound for spectral norm. However, $n\gtrsim mr$ is needed for stable estimation under Frobenius norm as in \cite{candesplan}. 
We are not sure whether $n\gtrsim mr^2$ is indeed required or some techiques are needed to get rid of one $r$. 
The following result is an immedate one from Lemma~\ref{corlem}.
\begin{cor}
 \label{corcor}
 When Assumption~\ref{corassump} is satisfied, for any $A$ and $B\in\mathbb{R}^{m\times m}$ with $\text{rank}(A)=\text{rank}(B)=1$ and $\left<A,B\right>=0$,
 \begin{equation}
  \frac{1}{n}\left|\left<\mathcal{X}(A),\mathcal{X}(B)\right>\right|\leq \frac{1}{\alpha(1+2c_0)r} 
 \end{equation}
\end{cor}
The next lemma shows that when $\lambda$ is able to hold the noise, $\hat{A}_{\lambda}-A_0$ belongs to some cone defined in Section~\ref{introsect}.
\begin{lem}
\label{conelem}
Take $\lambda>0$ such that $\lambda\geq 2||W||_{\infty}$, then we have
\begin{equation}
\label{conelemdantineq}
 \left|\left|\mathcal{X}^{\star}\mathcal{X}(\hat{A}_{\lambda}^d-A_0)\right|\right|_{\infty}\leq \frac{3\lambda}{2},\quad \hat{A}_{\lambda}^d-A_0\in\mathcal{C}(r,1)
\end{equation}
and
\begin{equation}
\label{conelemlassoineq}
 \left|\left|\mathcal{X}^{\star}\mathcal{X}(\hat{A}_{\lambda}^L-A_0)\right|\right|_{\infty}\leq \frac{3\lambda}{2},\quad \hat{A}_{\lambda}^L-A_0\in\mathcal{C}(r,3)
\end{equation}
\end{lem}
\begin{proof}
 We prove (\ref{conelemdantineq}) first. According to definition of $\hat{A}_{\lambda}^d$, we get that $\left|\left|\mathcal{X}^{\star}\mathcal{X}(\hat{A}_{\lambda}^d-A_0)-W\right|\right|_{\infty}\leq \lambda$.
 Then we get $\left|\left|\mathcal{X}^{\star}\mathcal{X}(\hat{A}_{\lambda}^d-A_0)\right|\right|_{\infty}\leq ||W||_{\infty}+\lambda\leq \frac{3}{2}\lambda$. 
 Since $||\hat{A}_{\lambda}^d||_1=||\hat{A}_{\lambda}^d-A_0+A_0||_1$ and, by Weilandt-Hoffman inequality, Tao\cite{taorandmat}, we get that
 \begin{equation}
 \begin{split}
  ||\hat{A}_{\lambda}^d||_1=&||\hat{A}_{\lambda}^d-A_0+A_0||_1\\
  \geq& \sum\limits_{j=1}^m\left|\sigma_j(\hat{A}_{\lambda}^d-A_0)-\sigma_j(A_0)\right|\\
  \geq& \sum\limits_{j=1}^r\left( \sigma_j(A_0)-\sigma_j(\hat{A}_{\lambda}^d-A_0)\right)+\sum\limits_{j=r+1}^m\sigma_j(\hat{A}_{\lambda}^d-A_0).
  \end{split}
 \end{equation}
Since $||\hat{A}_{\lambda}^d||_1\leq ||A_0||_1$, we can get $\sum\limits_{j=r+1}^m\sigma_j(\hat{A}_{\lambda}^d-A_0)\leq \sum\limits_{j=1}^r\sigma_j(\hat{A}_{\lambda}^d-A_0)$. Therefore, $\hat{A}_{\lambda}^d-A_0\in\mathcal{C}(r,1)$.
Now, we prove (\ref{conelemlassoineq}). According to standard convex optimization result, we know there exists some $\hat{V}\in\partial||\hat{A_{\lambda}^L}||_1$ such that
\begin{equation}
 \mathcal{X}^{\star}\mathcal{X}(\hat{A}_{\lambda}^L-A_0):=\sum\limits_{j=1}^n\left<\hat{A}_{\lambda}^L-A_0,X_j\right>X_j=\sum\limits_{j=1}^n\xi_jX_j-\lambda\hat{V}
\end{equation}
Since $||\hat{V}||_{\infty}\leq 1$, we get that $||\mathcal{X}^{\star}\mathcal{X}(\hat{A}_{\lambda}^L-A_0)||_{\infty}\leq ||W||_{\infty}+\lambda\leq\frac{3\lambda}{2}$. According to the definition of $\hat{A}_{\lambda}^L$, we have
\begin{equation}
 \sum\limits_{j=1}^n\left(\left<\hat{A}_{\lambda}^L,X_j\right>-Y_j\right)^2+\lambda||\hat{A}_{\lambda}^L||_1\leq \sum\limits_{j=1}^n\xi_j^2+\lambda||A_0||_1
\end{equation}
Therefore, we get that
\begin{equation}
\begin{split}
 \sum\limits_{j=1}^n\left<\hat{A}_{\lambda}^L-A_0,X_j\right>^2+&\lambda||\hat{A}_{\lambda}^L||_1\leq \left<\hat{A}_{\lambda}^L-A_0,W\right>+\lambda||A_0||_1\\
 \leq& ||\hat{A}_{\lambda}^L-A_0||_1||W||_{\infty}+\lambda||A_0||_1\\
 \leq& \frac{\lambda}{2}||\hat{A}_{\lambda}^L-A_0||_1+\lambda||A_0||_1
 \end{split}
\end{equation}
Which gives $2||\hat{A}_{\lambda}||_1\leq ||\hat{A}_{\lambda}-A_0||_1+2||A_0||_1$. Then we repeat the same process as above and we have
\begin{equation}
 ||\hat{A}_{\lambda}^L||_1\geq ||A_0||_1-\sum\limits_{j=1}^r\sigma_j(\hat{A}_{\lambda}^L-A_0)+\sum\limits_{j=r+1}^m\sigma_{j}(\hat{A}_{\lambda}^L-A_0)
\end{equation}
Since $||\hat{A}_{\lambda}^L||_1\leq ||A_0||_1+\frac{1}{2}||\hat{A}_{\lambda}-A_0||_1$, it is easy to get that,
\begin{equation}
 \frac{1}{2}\sum\limits_{j=r+1}^m\sigma_j(\hat{A}_{\lambda}^L-A_0)\leq \frac{3}{2}\sum\limits_{j=1}^r\sigma_j(\hat{A}_{\lambda}^L-A_0)
\end{equation}
Therefore, $\hat{A}_{\lambda}^L-A_0\in\mathcal{C}(r,3)$.
\end{proof}

\begin{lem}
\label{rsclem}
 Let Assumption~\ref{corassump} be satisfied. Then
 \begin{equation}
  \kappa(r,c_0):=\underset{\Delta\in\mathcal{C}(r,c_0)}{\min}\frac{||\mathcal{X}\Delta||_{l_2}}{\sqrt{n}||\Delta_{\max(r)}||_2}\geq c_1:=\sqrt{1-\frac{1}{\alpha}}>0.
 \end{equation}
\end{lem}
\begin{proof}
 Assume $\Delta\in\mathcal{C}(r,c_0)$ has singular value decomposition as $\Delta=\sum\limits_{j=1}^m\sigma_j(\Delta)u_j\otimes v_j$. We know that, based on Assumption~\ref{corassump} and Corollary~\ref{corcor}
 \begin{equation}
  \begin{split}
   \frac{1}{n}||\mathcal{X}\Delta_{\max(r)}||_{l_2}^2=&\frac{1}{n}\sum\limits_{j=1}^r\sigma_j^2(\Delta)\left<\mathcal{X}u_j\otimes v_j,\mathcal{X}u_j\otimes v_j\right>+
   \frac{1}{n}\sum\limits_{i\neq j}^r\sigma_i(\Delta)\sigma_j(\Delta)\left<\mathcal{X}u_i\otimes v_i,\mathcal{X}u_j\otimes v_j\right>\\	
   \geq& \left(1-\frac{1}{\alpha(1+2c_0)r}\right)\sum\limits_{j=1}^r\sigma_j^2(\Delta)-\frac{1}{\alpha(1+2c_0)r}\sum\limits_{i\neq j=1}^r\sigma_i(\Delta)\sigma_j(\Delta)\\
   =&||\Delta_{\max(r)}||_2^2-\frac{1}{\alpha(1+2c_0)r}||\Delta_{\max(r)}||_1^2
  \end{split}
 \end{equation}
Therefore, we get that
\begin{equation}
 \begin{split}
  \frac{1}{n}||\mathcal{X}\Delta||_{l_2}^2\geq&\frac{1}{n}\left<\mathcal{X}\Delta_{\max(r)},\mathcal{X}\Delta_{\max(r)}\right>+\frac{2}{n}\left<\mathcal{X}\Delta_{\max(r)},\mathcal{X}\Delta_{-\max(r)}\right>\\
  \geq&||\Delta_{\max(r)}||_2^2-\frac{1}{\alpha(1+2c_0)r}||\Delta_{\max(r)}||_1^2\\
  -&\frac{2}{\alpha(1+2c_0)r}||\Delta_{\max(r)}||_1||\Delta_{-\max(r)}||_1\\
  \geq&||\Delta_{\max(r)}||_2^2-\frac{1+2c_0}{\alpha(1+2c_0)r}||\Delta_{\max(r)}||_1^2\\
  \geq&\left(1-\frac{1}{\alpha}\right)||\Delta_{\max(r)}||_2^2
 \end{split}
\end{equation}
where the last inequality comes from the fact that $||\Delta_{\max(r)}||_1\leq \sqrt{r}||\Delta_{\max(r)}||_2$, since $\text{rank}(\Delta_{\max(r)})\leq r$.
\end{proof}
Now we state our main theorem as follows.
\begin{thm}
\label{mainthm}
 We choose $\lambda$ as in Lemma~\ref{conelem} and let Assumption~\ref{corassump} be satisfied, if $\text{rank}(A_0)\leq r$, then
 \begin{equation}
 \label{mainthmspectralineq}
  ||\hat{A}_{\lambda}^d-A_0||_{\infty}\leq c_d\frac{\lambda}{n},
 \end{equation}
 \begin{equation}
 \label{mainthmnuclearineq}
  ||\hat{A}_{\lambda}^d-A_0||_1\leq c'_d\frac{r\lambda}{n},
 \end{equation}
 and for any integer $1\leq k\leq m$,
 \begin{equation}
  ||\hat{A}_{\lambda}^d-A_0||_{F_k}\leq c_d(1+c_0)\frac{(k\wedge r)\lambda}{n},
 \end{equation}
 where $c_d=\frac{3}{2}+\frac{3(1+c_0)^2}{2\alpha(1+2c_0)c_1^2}$ and $c'_d=\frac{3(1+c_0)^2}{2c_1^2}$. (\ref{mainthmspectralineq}) and (\ref{mainthmnuclearineq}) are also true if we replace $\hat{A}_{\lambda}^d$ by $\hat{A}_{\lambda}^L$.
\end{thm}
\begin{proof}
  Our proof will use notation $\hat{A}_{\lambda}^d$, however, the method also works for $\hat{A}_{\lambda}^L$. According to Lemma~\ref{conelem}, we have $\left|\left|\mathcal{X}^{\star}\mathcal{X}(\hat{A}^d_{\lambda}-A_0)\right|\right|_{\infty}\leq \frac{3}{2}\lambda$.
 Let $\hat{\Delta}^d:=\hat{A}_{\lambda}^d-A_0=\sum\limits_{j=1}^m\sigma_j(\hat{\Delta}^d)u_j^d\otimes v_j^d$. Therefore, we get
 \begin{equation}
  \left<\mathcal{X}^{\star}\mathcal{X}\hat{\Delta}^d,u^d_1\otimes v^d_1\right>\leq \frac{3\lambda}{2}
 \end{equation}
However, we have that
\begin{equation}
 \begin{split}
   \left<\mathcal{X}^{\star}\mathcal{X}\hat{\Delta}^d,u^d_1\otimes v^d_1\right>=&\sum\limits_{j=1}^m\sigma_j(\hat{\Delta}^d)\left<\mathcal{X}u^d_j\otimes v^d_j,\mathcal{X}u^d_1\otimes v^d_1\right>\\
   \geq& n\left(1-\frac{1}{\alpha(1+2c_0)r}\right)\sigma_1(\hat{\Delta}^d)+\sum\limits_{j=2}^m\sigma_j(\hat{\Delta}^d)\left<\mathcal{X}u^d_j\otimes v^d_j,\mathcal{X}u^d_1\otimes v^d_1\right>\\
   \geq& n\left(1-\frac{1}{\alpha(1+2c_0)r}\right)\sigma_1(\hat{\Delta}^d)-\frac{n}{\alpha(1+2c_0)r}\sum\limits_{j=2}^m\sigma_j(\hat{\Delta}^d)
 \end{split}
\end{equation}
Therefore, we have $\sigma_1(\hat{\Delta}^d)\leq \frac{3\lambda}{2n}+\frac{1}{\alpha(1+2c_0)r}||\hat{\Delta}^d||_1$. Meanwhile, with $\hat{\Delta}^d\in\mathcal{C}(r,c_0)$, we have
\begin{equation}
\label{mainthmineq1}
 \left<\mathcal{X}^{\star}\mathcal{X}\hat{\Delta}^d,\hat{\Delta}^d\right>\leq \frac{3\lambda}{2}||\hat{\Delta}^d||_1 \leq \frac{3(1+c_0)}{2}\lambda||\Delta^d_{\max(r)}||_1\leq \frac{3(1+c_0)}{2}\lambda\sqrt{r}||\Delta^d_{\max(r)}||_2
\end{equation}
According to Lemma~\ref{rsclem}, we have $\left<\mathcal{X}^{\star}\mathcal{X}\hat{\Delta}^d,\hat{\Delta}^d\right>\geq nc_1^2||\hat{\Delta}^d_{\max(r)}||_2^2$. 
Together with (\ref{mainthmineq1}) we get $||\hat{\Delta}^d_{\max(r)}||_2\leq \frac{3(1+c_0)\lambda\sqrt{r}}{2nc_1^2}$. Therefore,
\begin{equation}
\begin{split}
 \sigma_1(\hat{\Delta}^d)\leq& \frac{3\lambda}{2n}+\frac{1+c_0}{\alpha(1+2c_0)r}||\hat{\Delta}_{\max(r)}^d||_1\\
 \leq&\frac{3\lambda}{2n}+\frac{3(1+c_0)^2\lambda r}{2\alpha(1+2c_0)rnc_1^2}\\
 =&\frac{\lambda}{n}\left(\frac{3}{2}+\frac{3(1+c_0)^2}{2\alpha(1+2c_0)c_1^2}\right)
 \end{split}
\end{equation}
Therefore, $\sigma_1(\hat{\Delta}^d)\leq c_d\frac{\lambda}{n}$. Meanwhile, $||\hat{\Delta}^d||_1\leq (1+c_0)||\hat{\Delta}_{\max(r)}^d||_1\leq (1+c_0)\sqrt{r}||\hat{\Delta}_{\max(r)}^d||_2$.
The upper bound for $||\hat{\Delta}^d||_{F_k}$ is just an immediate result. 
\end{proof}
Applying the interpolation inequality as (\ref{interpolationineq}), we get the following corollary.
\begin{cor}
\label{maincor}
 Under the same assumptions of Theorem~\ref{mainthm}, there exists some constant $C>0$ such that for every $1\leq q\leq\infty$,
 \begin{equation}
  \left|\left|\hat{A}_{\lambda}-A_0\right|\right|_{q}\leq C\frac{\lambda r^{1/q}}{n}
 \end{equation}
 where $\hat{A}_{\lambda}$ can be $\hat{A}_{\lambda}^d$ and $\hat{A}_{\lambda}^L$.
\end{cor}
\section{Spectral norm rate under sub-Gaussian measurements}
\label{subgaussiansect}
Based on the results in the previous section, we show the main theorem of this paper for sub-Gaussian measurements. Under sub-Gaussian measurements, we will see that Assumption~\ref{corassump} holds with high probability.
The following lemma is an immediate result from Proposition~\ref{empprop} in Appendix~\ref{appB}.
\begin{lem}
\label{subgaussianriplem}
 Suppose $X_j,j=1,\ldots,n$ are $i.i.d.$ sub-Gaussian measurements and $n\geq C_1m[\alpha^2(1+2c_0)^2r^2\vee \alpha(1+2c_0)r\log(m)\log(n)]$ for some $C_1>0$, then with probability at least $1-\frac{1}{m}$,
 Assumption~\ref{corassump} holds.
\end{lem}

The following lemma provides a choice of $\lambda$.
The proof is given in Appendix~\ref{appA}.
\begin{lem}
 \label{bernsteinlem}
 Under the assumption that $n\geq C_1m\log(m)\log(n)$ for some $C_1>0$, if $|\xi|_{\psi_2}\lesssim \sigma_{\xi}$ and $\Pi$ is a sub-Gaussian distribution, then for every $t>0$, with probability at least $1-2e^{-t}-\frac{1}{m}$ we have
 \begin{equation}
 \label{bernsteinlemineq1}
  \left|\left|\frac{1}{n}\sum\limits_{j=1}^n\xi_jX_j\right|\right|_{\infty}\leq C\sigma_{\xi}\sqrt{\frac{mt}{n}}
 \end{equation}
for some constant $C>0$, where $C$ contains constant related to $\Pi$.
\end{lem}

Now, we state the sub-Gaussian version of Theorem~\ref{mainthm}.
\begin{thm}
\label{mainthmsubgaussian}
 Suppose $X_j,j=1,\ldots,n$ are $i.i.d.$ sub-Gaussian measurements,$|\xi|_{\psi_2}\lesssim \sigma_{\xi}$ and any $\alpha>1	$, $c_0=1$ for Dantzig Selector, $c_0=3$ for matrix LASSO estimator. There exists some constants $C_1,C_2>0$ such that when $n\geq C_1m[\alpha^2(1+2c_0)^2r^2\vee \alpha(1+2c_0)r\log(m)\log(n)]$ and 
 $\lambda:=C_2\sigma_{\xi}\sqrt{mn\log(m)}$, with probability at least $1-\frac{4}{m}$,
 \begin{equation}
  \left|\left|\hat{A}_{\lambda}-A_0\right|\right|_{\infty}\leq c_dC_2\sigma_{\xi}\sqrt{\frac{m\log(m)}{n}}
 \end{equation}
and
\begin{equation}
  \left|\left|\hat{A}_{\lambda}-A_0\right|\right|_1\leq c'_dC_2\sigma_{\xi}\sqrt{\frac{m\log(m)}{n}}
 \end{equation}
 and for any integer $1\leq k\leq m$,
 \begin{equation}
  \left|\left|\hat{A}_{\lambda}-A_0\right|\right|_{F_k}\leq c_d(1+c_0)C_2\sigma_{\xi}(r\wedge k)\sqrt{\frac{m\log(m)}{n}}
 \end{equation}
 where $\hat{A}_{\lambda}$ can be $\hat{A}_{\lambda}^d$ or $\hat{A}_{\lambda}^L$ with only $c_0$ different and $c_d,c'_d$ are the same as Theorem~\ref{mainthm}.
\end{thm}
\begin{proof}
 According to Lemma~\ref{subgaussianriplem}, Assumption~\ref{corassump} is satisfied with probability at least $1-\frac{1}{m}$.
 With $C_2$ well chosen, we see that $\lambda\geq 2||W||_{\infty}$ holds with probability at least $1-\frac{3}{m}$ from Lemma~\ref{bernsteinlem}. Therefore, based on Lemma~\ref{conelem} and Theorem~\ref{mainthm},
 we can get our desired bound. 
\end{proof}
The Theorem~\ref{intromainthm2} is a direct result of Theorem~\ref{mainthmsubgaussian} by applying Corollary~\ref{maincor}.
\section{Minimax Lower Bound}
\label{minimaxsect}
In this section, we will prove Theorem~\ref{intromainthm3}. In an earlier paper by Ma and Wu\cite{mawu}, they provided similar minimax lower bounds for more general norms by using volume ratios. 
Our method constructs a well-seperated set of low rank matrices by applying the metric entropy bounds of Grassmann manifold. Suppose $2r\leq m$, consider any $1\leq q\leq \infty$, by Proposition~\ref{metricprop} and inequality (\ref{coverpackineq}),
we know that $M(\mathcal{G}_{m,r},\tau_q,\epsilon r^{1/q})\geq \left(\frac{c}{\epsilon}\right)^{r(m-r)}$. Therefore, there exists a set $\mathcal{B}:=\left\{P_{B_j}: B_j\in\mathcal{G}_{m,r}\right\}$ 
with $\text{card}(\mathcal{B})\geq 2^{r(m-r)}$ and $\tau_q(P_{B_j}-P_{B_k})\geq \frac{c}{2}r^{1/q}$ for any $j\neq k$. Based on $\mathcal{B}$, we construct the following set:
$\mathcal{A}:=\left\{\kappa P_{B_j}: P_{B_j}\in\mathcal{B}\right\}$ with $\kappa=c'\sigma_{\xi}\sqrt{\frac{m}{n}}$ with a small positive constant $c'>0$ which will be determined later.\\
For any $A_j,A_k\in\mathcal{A}, j\neq k$, we know that $\tau_q(A_j-A_k)\geq \frac{c\kappa}{2}r^{1/q}$. When $\xi\sim\mathcal{N}(0,\sigma_{\xi}^2)$, we can get for any $A,B\in\mathbb{R}^{m\times m}$,
\begin{equation}
 \begin{split}
 K(\mathbb{P}_A||\mathbb{P}_B)&=\mathbb{E}_{\mathbb{P}_A}\left[\text{log}\frac{\mathbb{P}_A}{\mathbb{P}_B}(X_1,Y_1,\ldots,X_n,Y_n) \right]\\
&=\mathbb{E}_{\mathbb{P}_A}\left[\sum\limits_{j=1}^n\left(-\frac{(Y_i-\left<A,X_i\right>)^2}{2\sigma_\xi^2}+\frac{(Y_i-\left<B,X_i\right>)^2}{2\sigma_\xi^2} \right) \right]\\
&=\mathbb{E}_{\mathbb{P}_A|\Pi}\mathbb{E}_{\Pi}\left[\sum\limits_{i=1}^n \frac{\left<A-B,X_i\right>(2Y_i-\left<A+B,X_i\right>)}{2\sigma_\xi^2}\right]\\
&=\frac{n}{2\sigma_\xi^2}||A-B||_{L_2(\Pi)}^2\lesssim \frac{n}{\sigma_{\xi}^2}||A-B||_2^2
\end{split}
\end{equation}
where $\mathbb{P}_A$ denotes the joint distribution of $(X_1,Y_1),\ldots,(X_n,Y_n)$ when $Y_j=\left<A,X_j\right>+\epsilon_j, j=1,\ldots,n$. The last inequality holds because $\Pi$ is a sub-Gaussian distribution.
From this inequality, we know that for any $A_j,A_k\in\mathcal{A},j\neq k$,
\begin{equation}
 K(\mathbb{P}_{A_j}||\mathbb{P}_{A_k})\lesssim \frac{n}{\sigma^2_{\xi}}||A_j-A_k||_2^2\leq \frac{2n\kappa^2r}{\sigma^2_{\xi}}=2c'mr\leq (mr-r^2)\log2\leq\log(\text{card}(\mathcal{A}))
\end{equation}
The third inequality holds whenever $c'$ is small enough. Then Theorem~\ref{intromainthm3} is an immediate conclusion by applying Tsybakov\cite[Theorem 2.5]{intro}. Indeed, by applying \cite[Theorem 2.5]{intro}, we have
\begin{equation}
  \underset{\hat{A}}{\inf}\underset{A\in\mathcal{A}_r}{\sup}\mathbb{P}_A\left(||\hat{A}-A||_q\geq c\sigma_{\xi}r^{1/q}\sqrt{\frac{m}{n}}\right)\geq c'
\end{equation}
for certain $c,c'>0$. The minimax lower bound for Ky-Fan-k norm is similar by choosing $r=k$ and $q=1$.

\section{Numerical Simulations}
\label{numericsect}
In this section, we show the results of numerical simulations. Since (\ref{lassomodel}) and (\ref{dantmodel}) are equivalent for certain $\lambda>0$, we only implement
numerical experiments for $\hat{A}_{\lambda}^L$. I should point out that even our analysis for optimal upper bound of $||\hat{A}_{\lambda}^L-A_0||_{\infty}$ requires 
that $n\gtrsim mr^2$, our numerical experiments will show that $n\gtrsim mr$ is indeed enough. To solve the optimization problem (\ref{lassomodel}), we will implement the Alternating
Direction Method of Multipliers(ADMM), Boyd et. al.\cite{boyd}, Lin et. al.\cite{linchen}. (\ref{lassomodel}) is equivalent to the following optimization problem:
\begin{equation}
 \hat{A}_{\lambda}:=\underset{A=B\in\mathbb{R}^{m\times m}}{\arg\min}\sum\limits_{j=1}^n\left(\left<A,X_j\right>-Y_j\right)^2+\lambda||B||_1
\end{equation}
ADMM forms the augmented Lagrangian:
\begin{equation}
 L_{\rho}(A,B,Z):=\sum\limits_{j=1}^n\left(\left<A,X_j\right>-Y_j\right)^2+\lambda||B||_1+\left<Z,A-B\right>+\frac{\rho}{2}||A-B||_2^2
\end{equation}
ADMM consists of the iterations as in Algorithm~\ref{admm_const}. Many papers in the literature showed that ADMM has good covergence properties. 
In our numcerical experiments, we choose
$n=5mr$ and $\lambda=7\sigma_{\xi}\sqrt{mn}$, where we fixed $\sigma_{\xi}=0.01$ and $\xi\sim\mathcal{N}(0,\sigma^2_{\xi})$. The low rank matrix $A_0$ is constructed as
a product of a $m*r$ Gaussian matrix and a $r*m$ Gaussian matrix. In our experiments, we implemented $m=40,50,60$ and $3\leq r\leq 25$, with $5$ trials for every $m$ and $r$. The measurements $X_1,\ldots,X_n$ are 
random Gausisan matrices or Rademacher matrices.
\begin{algorithm}
\label{admmalgorithm}
 \caption{ADMM Algorithm}\label{admm_const}
  \begin{algorithmic}[3]
  \State Set up value of $\text{max}\_\text{Iteration}$ and tolerance $\epsilon_{\text{tol}}>0$
   \State Initiate random $A^{(0)}\in\mathbb{R}^{m\times m}$, $B^{(0)}\in\mathbb{R}^{m\times m}$ and $Z^{(0)}=\bf{0}\in\mathbb{R}^{m\times m}$, k=0
   \While{k$<$max\_Iteration}
   \State $A^{(k+1)}=\underset{A\in\mathbb{R}^{m\times m}}{\arg\min}\quad \sum\limits_{j=1}^n\left(Y_j-\left<A,X_j\right>\right)^2+\left<A-B^{(k)},Z^{(k)}\right>+\frac{\rho}{2}||A-B^{(k)}||_2^2$
   \State $B^{(k+1)}=\underset{B\in\mathbb{R}^{m\times m}}{\arg\min}\quad \lambda||B||_1+\left<A^{(k+1)}-B,Z^{(k)}\right>+\frac{\rho}{2}||A^{(k+1)}-B||_2^2$\\
   \State $Z^{(k+1)}=Z^{(k)}+\rho(A^{(k+1)}-B^{(k+1)})$
   \If{$||A^{(k+1)}-B^{(k+1)}||_2^2\leq\epsilon_{\text{tol}}$}
      \State Reaching the tolerance. Return $A^{(k+1)}$ or $B^{(k+1)}$.
   \EndIf
   \State k=k+1
   \EndWhile
  \State Return $A^{(k+1)}$ or $B^{(k+1)}$.
  \end{algorithmic}
\end{algorithm}
The numerical results in Figure~\ref{fig:fig1} shows that under Gaussian measurements, we have $||\hat{A}_{\lambda}^{L}-A_0||_{\infty}\sim C\sigma_{\xi}\sqrt{\frac{m}{n}}$
where $C$ is betwwen $8$ and $10$. Since we choose $n=5mr$, $i.e.$, $\sigma_{\xi}\sqrt{\frac{m}{n}}\sim\sigma_{\xi}\sqrt{\frac{1}{r}}$, (\ref{fig:fig1_sub1}) in Figure~\ref{fig:fig1} shows that $||\hat{A}_{\lambda}^L-A_0||_{\infty}$ depends only
on the rank of $A_0$. 
\begin{figure}
\centering
\begin{subfigure}{.5\textwidth}
  \centering
  \includegraphics[height=3in]{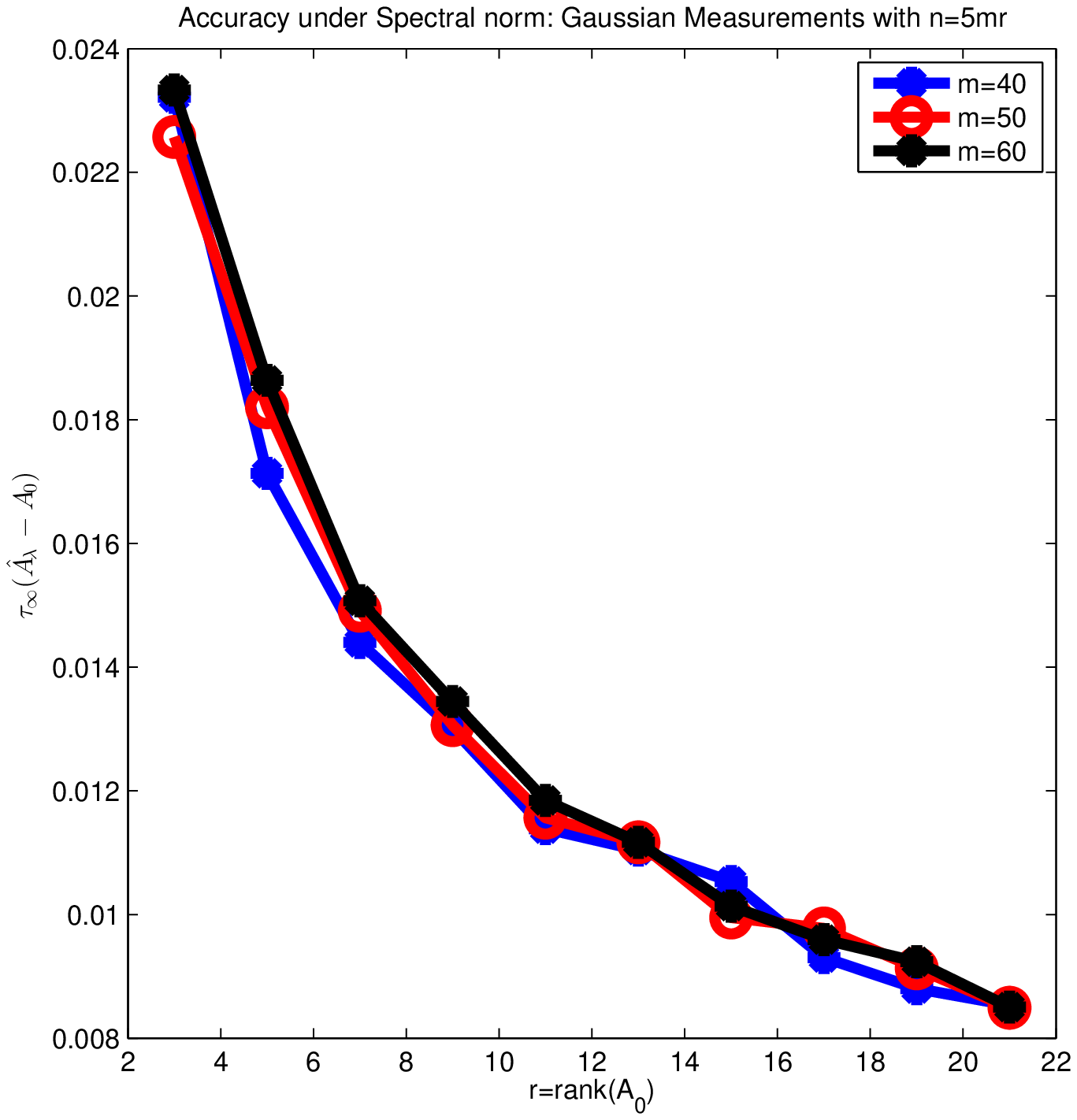}
  \caption{Accuracy under Spectral Norm}
  \label{fig:fig1_sub1}
\end{subfigure}%
\begin{subfigure}{.5\textwidth}
  \centering
  \includegraphics[height=3in]{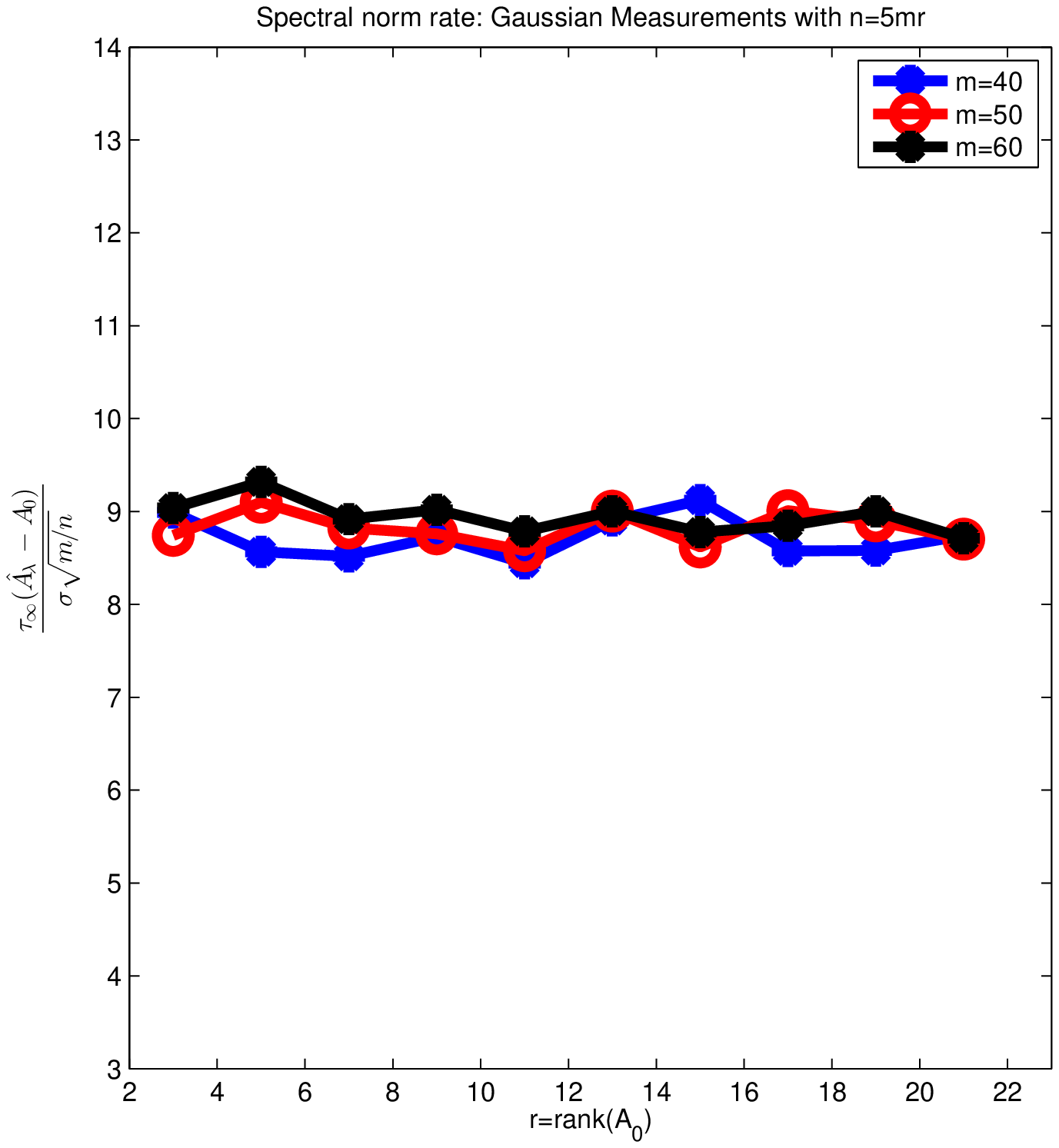}
  \caption{Accuracy ratio under Spectral Norm}
  \label{fig:fig1_sub2}
\end{subfigure}
\caption{Accuracy (ratio) by Spectral Norm under Gaussian measurements. The number of measurements is $n=5mr$ with $3\leq r\leq 21$ and $m=40,50,60$.
The x-axis stands for the $\text{rank}(A_0)$. In (\ref{fig:fig1_sub1}), the y-axis represents the average loss (5 trials) under spectral norm, $i.e., ||\hat{A}_{\lambda}^L-A_0||_{\infty}$.
We see that the average loss decreases with $\text{rank}(A_0)$ increases. In (\ref{fig:fig1_sub2}), the y-axis represents the ratio between the simulation accuracy and theoretical order of accuracy, $i.e., \frac{||\hat{A}_{\lambda}^L-A_0||_{\infty}}{\sigma_{\xi}\sqrt{m/n}}$.
It shows that the ratio belongs to $[8,10]$, remember that we choose $\lambda=7\sigma_{\xi}\sqrt{m/n}$.
}
\label{fig:fig1}
\end{figure}

In Figure~\ref{fig:fig2}, we show the behavior of accuracy by Spectral norm under Rademacher measurements. Similar to the results of Figure~\ref{fig:fig1},
estimation accuracy decreases as $\text{rank}(A_0)$ increases. 
\begin{figure}
\centering
\begin{subfigure}{.5\textwidth}
  \centering
  \includegraphics[height=3in]{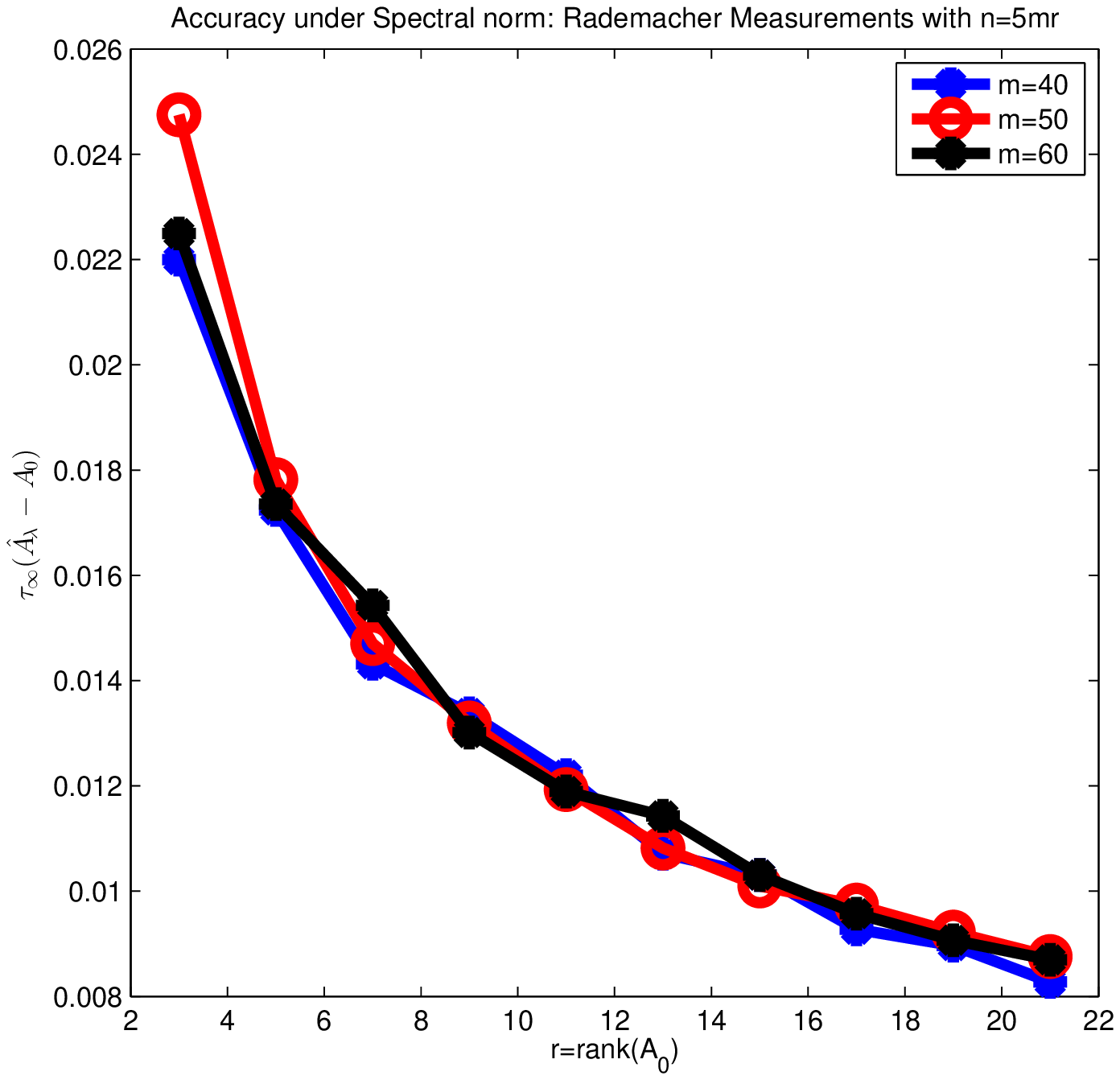}
  \caption{Accuracy under Spectral norm}
  \label{fig:fig2_sub1}
\end{subfigure}%
\begin{subfigure}{.5\textwidth}
  \centering
  \includegraphics[height=3in]{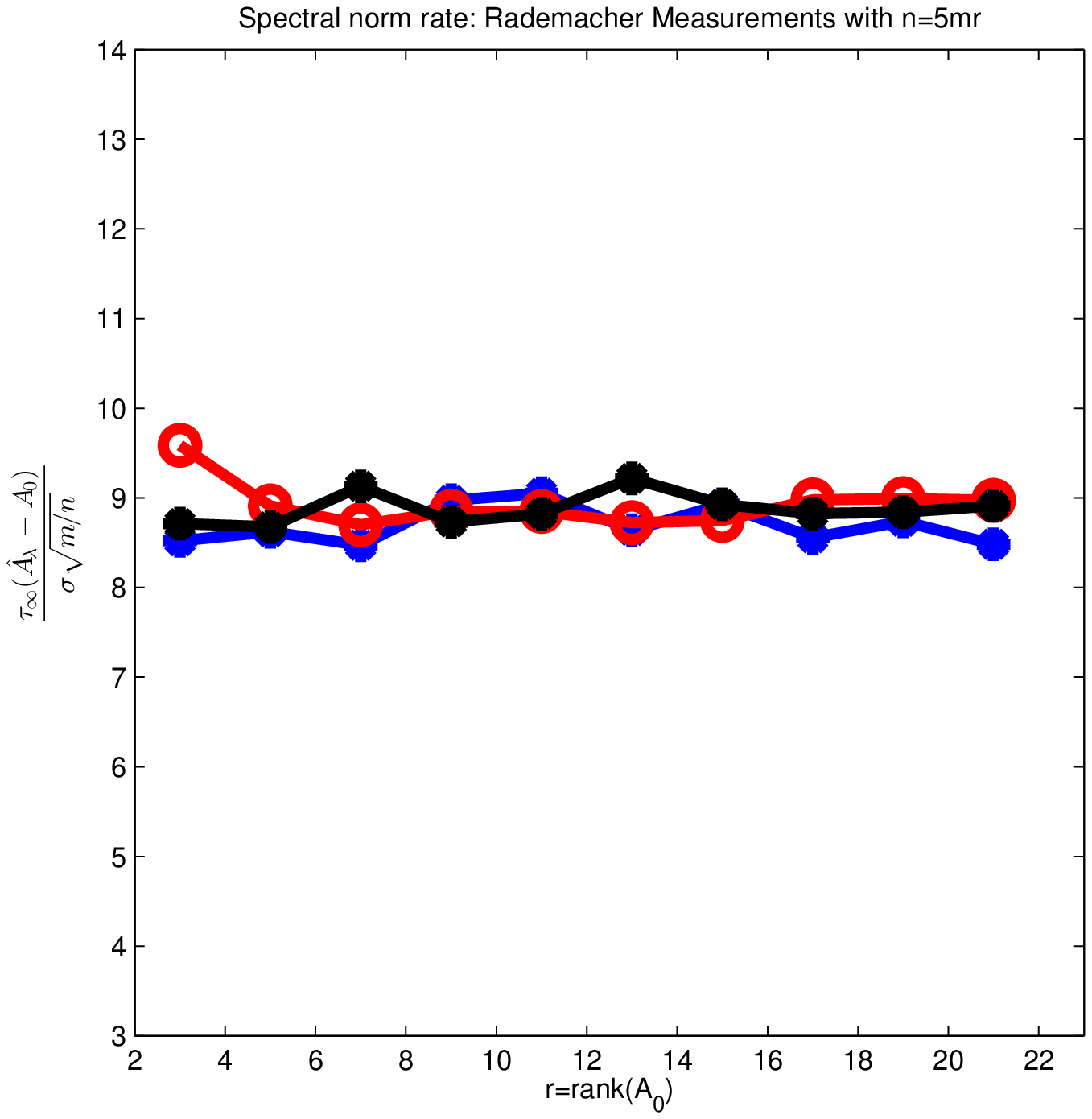}
  \caption{Accuracy ratio under Spectral norm}
  \label{fig:fig2_sub2}
\end{subfigure}
\caption{Accuracy by Spectral Norm (ratio) under Rademacher measurements. Similar to the results in Figure~\ref{fig:fig1}), the ratio $\frac{||\hat{A}_{\lambda}^L-A_0||_{\infty}}{\sigma_{\xi}\sqrt{m/n}}\in[8,10]$.
}
\label{fig:fig2}
\end{figure}
\newpage
\appendix

\section{Proof of Lemma~\ref{bernsteinlem}}\label{appA}
A well-known fact is as follows, Rudelson and Vershynin\cite{rudelsonvershynin}.
\begin{prop}
\label{subgaussianspecprop}
 Let $X\in\mathbb{R}^{m\times m}$ be a sub-Gaussian matrix. There exists a constant $B>0$ such that $||||X||_{\infty}||_{\psi_2}\leq B\sqrt{m}$.
\end{prop}
\begin{proof}[{\it Proof of Lemma~\ref{bernsteinlem}}]
 We consider sub-Gaussian noise, $i.e.$, $|\xi|_{\psi_2}\lesssim \sigma_{\xi}$.
 We know that, 
 $$
 \left|\left|\frac{1}{n}\sum\limits_{j=1}^n\xi_j X_j\right|\right|_{\infty}=\underset{||U||_2\leq1,\text{rank}(U)\leq 1}{\sup}\frac{1}{n}\sum\limits_{j=1}^n\xi_j\left<X_j,U\right>.
 $$
 Accroding to Proposition~\ref{empprop}, there exists an event $\mathcal{E}$ with $\mathbb{P}(\mathcal{E})\geq 1-\frac{1}{m}$ such that $X_1,\ldots,X_n$ satisfy the RIP with $\delta_2\leq\frac{1}{2}$, as long as $n\geq Cm\log(m)\log(n)$ for some $C>0$.
 Now we fix $X_1,\ldots,X_n$ on the event $\mathcal{E}$ and define $K_U:=\frac{1}{\sqrt{n}}\sum\limits_{j=1}^n\xi_j\left<X_j,U\right>$ for $\forall U\in\mathcal{U}_1:=\left\{U\in\mathbb{R}^{m\times m}: ||U||_2\leq 1,\text{rank}(U)\leq 1\right\}$. This is a sub-Gaussian process indexed by $U\in\mathcal{U}_1$, since $\xi_j,j=1,\ldots,n$ are $i.i.d.$ sub-Gaussians. 
 For any $U,V\in\mathcal{U}_1$, we have $K_U-K_V=\frac{1}{\sqrt{n}}\sum\limits_{j=1}^n\xi_j\left<X_j,U-V\right>$. We konw that $\xi_j\left<X_j,U-V\right>,j=1,\ldots,n$ are $i.i.d.$ centered sub-Gaussian 
 random variables and $|\xi_j\left<X_j,U-V\right>|_{\psi_2}\lesssim \sigma_{\xi}|\left<X_j,U-V\right>|$ for $j=1,\ldots,n$. We apply Lemma~\ref{subgaussiannormlem} and get
 $||K_U-K_V||_{\psi_2}^2\lesssim \sigma_{\xi}^2\frac{1}{n}\sum\limits_{j=1}^n\left<X_j,U-V\right>^2\leq \sigma_{\xi}^2(1+\delta_2)||U-V||_2^2\leq 2\sigma_{\xi}^2||U-V||_2^2$. 
 Therefore, for every $U,V\in\mathcal{U}_1$, we have
 \begin{equation}
  ||K_U-K_V||_{\psi_2}\lesssim \sigma_{\xi}||U-V||_2
 \end{equation}
By defining a distance $d(U,V):=\sigma_{\xi}||U-V||_2$ for any $U,V\in\mathcal{U}_1$, we apply van der Vaart and Wellner\cite[Corollary 2.2.6]{vaartwellner}, we get that
\begin{equation}
 ||\underset{U\in\mathcal{U}_1}{\sup}K_U||_{\psi_2}\lesssim \int_0^{\text{diam}(\mathcal{U}_1)}\sqrt{\log M(\mathcal{U}_1,d,\epsilon)}d\epsilon
\end{equation}
where $\text{diam}(\mathcal{U}_1)=\underset{U,V\in\mathcal{U}_1}{\sup}d(U,V)\leq \sqrt{2}\sigma_{\xi}$. It is easy to see that $M(\mathcal{U}_1,d,\epsilon)\leq M(\mathcal{G}_{m,1},d,\epsilon)$.
According to Lemma~\ref{metricprop}, we know that $\log M(\mathcal{U}_1,d,\epsilon)\leq \log M(\mathcal{G}_{m,1},\tau_2,\frac{\epsilon}{\sigma_{\xi}})\leq m\log(\frac{C\sigma_{\xi}}{\epsilon})$.
Put these bounds into the integral, we get that
\begin{equation}
 ||\underset{U\in\mathcal{U}_1}{\sup}K_U||_{\psi_2}\lesssim \sqrt{m}\int_0^{2\sigma_{\xi}}\sqrt{\log(C\sigma_{\xi}/\epsilon)}d\epsilon\leq \sqrt{m}\sigma_{\xi}\int_{1/2}^{\infty}\frac{\sqrt{\log(Cu)}}{u^2}du\lesssim \sqrt{m}\sigma_{\xi}
\end{equation}
Therefore, we know that $||\underset{U\in\mathcal{U}_1}{\sup}K_U||_{\psi_2}\lesssim \sqrt{m}\sigma_{\xi}$. Therefore, for some $C_1>0$ and for every $\rho,t>0$,
\begin{equation}
\begin{split}
 \mathbb{P}\left(\underset{U\in\mathcal{U}_1}{\sup}K_U\geq C_1t\sqrt{m}\sigma_{\xi}\right)=&\mathbb{P}\left(\exp\{(\rho\underset{U\in\mathcal{U}_1}{\sup}K_U)^2\}\geq \exp\{C_1^2\rho^2t^2m\sigma_{\xi}^2\}\right)\\
 \leq&\exp\{-C_1^2\rho^2t^2m\sigma_{\xi}^2\}\mathbb{E}\exp\{(\underset{U\in\mathcal{U}_1}{\sup}K_U)^2\rho^2\}
 \end{split}
\end{equation}
We can choose $\rho\lesssim \frac{1}{\sqrt{m}\sigma_{\xi}}$ such that $\mathbb{E}\exp\{(\underset{U\in\mathcal{U}_1}{\sup}K_U)^2\rho^2\}\leq 2$ and we get that
$$
\mathbb{P}\left(\underset{U\in\mathcal{U}_1}{\sup}K_U\geq C_1t\sqrt{m}\sigma_{\xi}\right)\leq 2\exp\{-C_2t^2\},
$$
for some $C_2>0$.
By the definition of $K_U$, we get our desired bound. Since our analysis is conditioned on the event $\mathcal{E}$, there is an additional $\frac{1}{m}$.
 
\end{proof}

\section{An Empirical Process Bound}\label{appB}
\begin{prop}
\label{empprop}
 Suppose $X_1,\ldots,X_n$ are $i.i.d.$ sub-Gaussian matrices with distribution $\Pi$. Then, for an integer $1\leq r\leq m$ and all matrix $A$ with $||A||_2^2\in\left[\frac{1}{2},2\right]$, we have that for every $t>0$, with probability at least $1-e^{-t}$,
\begin{equation}
 \left|\frac{1}{n}\sum\limits_{j=1}^n\left<A,X_j\right>^2-\mathbb{E}\left<A,X\right>^2\right|\leq C||A||_2^2\left(\sqrt{\frac{t}{n}}\vee \frac{mrt\log(n)}{n}\vee \frac{rm}{n}\vee\sqrt{\frac{rm}{n}}\right)
\end{equation}
where $C>0$ is a universal constant related to $\Pi$
\end{prop}
\begin{proof}
We consider the following empirical process:
\begin{equation}
 \alpha_n(r,T)=\underset{A\in \Delta_r(T)}{\sup}\left|\frac{1}{n}\sum\limits_{j=1}^n\left<A,X_j\right>^2-\mathbb{E}\left<A,X_j\right>^2\right|
\end{equation}
where the set 
\begin{equation}
 \Delta_{r}(T):=\Big\{A\in\mathbb{R}^{m\times m}: \frac{T}{4}\leq ||A||_2^2\leq T, \text{rank}(A)\leq r\Big\}
 \end{equation}
We want to obtain an upper bound of $\alpha_n(r,T)$.

According to the Adamczak's version of Talagrand inequality (\ref{adamczakineq}), there exists some constant $K>0$ such that for any $t>0$, with probability at least $1-e^{-t}$,
\begin{equation}
 \alpha_n(r,T)\leq K\left[\mathbb{E}\alpha_n(r,T)+T\sqrt{\frac{t}{n}}+\frac{Tmrt\log(n)}{n}\right]
\end{equation}
Here we used the following bounds:
\begin{equation}
\begin{split}
 \underset{A\in\Delta_r(T)}{\sup}&\mathbb{E}\left<A,X\right>^4\lesssim \underset{A\in\Delta_{r}(T)}{\sup}||A||_2^4\leq T^2
 \end{split}
\end{equation}
where the first inequality comes from the fact $\mathbb{E}^{1/p}\left<A,X\right>^p\lesssim ||A||_2$ for $p\geq 1$ as introduced in Section~\ref{introsect}.
Meanwhile,
\begin{equation}
\begin{split}
 \Bigg|\Bigg|\underset{1\leq i\leq n}{\max}&\underset{A\in\Delta_{r}(T)}{\sup}\left<A,X\right>^2\Bigg|\Bigg|_{\psi_1}\lesssim \Bigg|\Bigg|\underset{A\in\Delta_{r}(T)}{\sup}\left<A,X\right>^2\Bigg|\Bigg|_{\psi_1}\log n\\
 \lesssim &rT||||X||_{\infty}^{2}||_{\psi_1}\log(n)\lesssim rT||||X||_{\infty}||_{\psi_2}^2\log(n)\lesssim rTm\log(n)
\end{split}
 \end{equation}
where we used well-known inequalities for maxima of random variables in Orlicz spaces, van der Varrat and Wellner\cite[Chapter 2]{vaartwellner}, and $||A||_1\leq \sqrt{r}||A||_2$ for any $A\in\Delta_r(T)$.\\
Now we try to get an upper bound for $\mathbb{E}\alpha_{n}(r,T)$. We apply Mendelson's inequality (\ref{mendelsonineq}) for the class of functions $\mathcal{F}_r(T):=\left\{f_A(\cdot):=\left<A,\cdot\right>: A\in\Delta_r(T)\right\}$.

According to the property of sub-Gaussiam matrices introduced in Section~\ref{introsect}, we know that $|f_A(X)|_{\psi_1}\lesssim ||A||_2\leq \sqrt{T}$. We also konw the following bound for Talagrand's generic chaining complexities in Orclize space:
\begin{equation}
 \gamma_2(\mathcal{F}_r(T);\psi_2)\leq \gamma_2(\mathcal{F}_r(T);c||\cdot||_2).
\end{equation}
where $c>0$ is a constant, since $||A||_{\psi_2}\lesssim ||A||_{L_2(\Pi)}\lesssim ||A||_2$ as introduced in Section~\ref{introsect}. From Talagrand's generic chaining bound, we get that
\begin{equation}
 \gamma_2(\mathcal{F}_r(T);c||\cdot||_2)\lesssim \mathbb{E}\underset{A\in\Delta_r(T)}{\sup}\left|\left<A,G\right>\right|\leq \sqrt{r}T\mathbb{E}||G||_{\infty}\lesssim \sqrt{rTm}
\end{equation}
where $G\in\mathbb{R}^{m_1\times m_2}$ denotes standard Gaussian matrix. The last inequality comes from the fact that 
\begin{equation}
 \mathbb{E}||G||_{\infty}\leq \mathbb{E}^{1/2}||G||_{\infty}^2\leq ||||G||_{\infty}||_{\psi_2}\sqrt{\log2}\lesssim \sqrt{m}
\end{equation}
where the first inequality comes from Jensen inequality and the last inequality comes from Proposition~\ref{subgaussianspecprop}. For the second inequality, by the definition of $||||G||_{\infty}||_{\psi_2}$,
\begin{equation}
 \mathbb{E}\exp\left\{||G||_{\infty}^2/||||G||_{\infty}||_{\psi_2}^2\right\}-1\leq 1
\end{equation}
By Jensen inequality, we get $\mathbb{E}||G||_{\infty}^2/||||G||_{\infty}||_{\psi_2}^2\leq \log2$.
Put these bound into (\ref{mendelsonineq}), we get
\begin{equation}
 \mathbb{E}\underset{f\in\mathcal{F}_r(T)}{\sup}\left|\frac{1}{n}\sum\limits_{j=1}^nf(X_j)^2-\mathbb{E}f^2(X)\right|\leq CT\left(\sqrt{\frac{rm}{n}}\vee \frac{rm}{n}\right)
\end{equation}
Therefore, we get that with probability at least $1-e^{-t}$ such that for some constant $C>0$,
\begin{equation}
 \underset{A\in\Delta_r(T)}{\sup}\left|\frac{1}{n}\sum\limits_{j=1}^n\left<A,X_j\right>^2-\mathbb{E}\left<A,X\right>^2\right|\leq C||A||_2^2\left(\sqrt{\frac{t}{n}}\vee \frac{mrt\log(n)}{n}\vee \frac{rm}{n}\vee\sqrt{\frac{rm}{n}}\right)
\end{equation}
\end{proof}











\bibliographystyle{abbrv}
\bibliography{refer}

 



\end{document}